\documentclass[moor,nonblindrev]{informs3opt}

\OneAndAHalfSpacedXI 

\usepackage[usestackEOL]{stackengine}
\usepackage{url}            
\usepackage{booktabs}       
\usepackage{amsfonts}       
\usepackage{nicefrac}       
\usepackage{microtype}      
\usepackage{amssymb}
\usepackage{array}
\usepackage{mdwmath}
\usepackage{mdwtab}
\usepackage{eqparbox}
\usepackage{algorithm}
\usepackage{algpseudocode}
\usepackage{algorithmicx}
\DeclareMathAlphabet{\mathpzc}{OT1}{pzc}{m}{it}

\usepackage{enumerate}
\usepackage[square,numbers]{natbib}
\usepackage{graphicx}

\usepackage{epstopdf}
\usepackage{multirow}

\usepackage[colorlinks,linkcolor=red,
filecolor=black,
citecolor = blue,      
urlcolor=black,]{hyperref}

\usepackage{amsmath}
\usepackage{bm} 
\usepackage{stmaryrd}
\usepackage{mathtools}
\usepackage{authblk}
\usepackage{blindtext,cleveref}

\let\theoremstyle\relax
\usepackage{amsthm} 
\newtheoremstyle{exampstyle}
{4pt} 
{4pt} 
{\normalfont\TheoremTextFont} 
{} 
{\bfseries} 
{} 
{.5em} 
{} 

\theoremstyle{exampstyle} \newtheorem{lemma}{Lemma}
\theoremstyle{exampstyle} 
\theoremstyle{exampstyle} \newtheorem{corollary}{Corollary}
\theoremstyle{exampstyle} \newtheorem{theorem}{Theorem}
\theoremstyle{exampstyle} 
\theoremstyle{exampstyle} 
\theoremstyle{exampstyle} 
\theoremstyle{exampstyle} 
\theoremstyle{exampstyle} \newtheorem{remark}{Remark}
\theoremstyle{exampstyle} \newtheorem{definition}{Definition}
\theoremstyle{exampstyle}

\usepackage{threeparttable}

\begin{document}
	
	
	\RUNAUTHOR{Shen, Kong }
	
	\RUNTITLE{Robust Tensor Principal Component Analysis: Exact Recovery via Deterministic Model}
	
	\TITLE{Robust Tensor Principal Component Analysis: Exact Recovery via Deterministic Model}
	

	\ARTICLEAUTHORS{%
		\AUTHOR{Bo Shen\textsuperscript{1}, Yutong Zhang\textsuperscript{2}, Zhenyu (James) Kong\textsuperscript{3}}
		\AFF{\EMAIL{\textsuperscript{1}\href{mailto:bo.shen@njit.edu}{bo.shen@njit.edu}, \textsuperscript{2}\href{mailto:ytzhang202@gmail.com}{ytzhang202@gmail.com}, \textsuperscript{3}\href{mailto:zkong@vt.edu}{zkong@vt.edu}}} 
	} 
	\ABSTRACT{Tensor, also known as multi-dimensional array, arises from many applications in signal processing, manufacturing processes, healthcare, among others. As one of the most popular methods in tensor literature, Robust tensor principal component analysis (RTPCA) is a very effective tool to extract the low rank and sparse components in tensors. In this paper, a new method to analyze RTPCA is proposed based on the recently developed tensor-tensor product and tensor singular value decomposition (t-SVD). Specifically, it aims to solve a convex optimization problem whose objective function  is a weighted combination of the tensor nuclear norm and the $\ell1$-norm. In most of literature of RTPCA, the exact recovery is built on the \textit{tensor incoherence conditions} and the assumption of a uniform  model on the sparse support. Unlike this conventional way, in this paper, without any assumption of randomness, the exact recovery can be achieved in a completely deterministic fashion by characterizing the \textit{tensor rank-sparsity incoherence}, which is  an uncertainty principle between  the  low-rank tensor spaces and the pattern of sparse tensor.}
	\maketitle
\section{Introduction}
Tensor, also known as multidimensional array, recently has attracted much attention in data analytics since real data can be naturally described as a tensor \cite{sidiropoulos2017tensor}. For example, RGB images are 3-way tensors since they have three channels.  As a tensor extension of robust principal component analysis (RPCA) \cite{candes2011robust,chandrasekaran2011rank}, 
\begin{figure}[!htbp] \label{fig: RTPCA illstruation}
\centering
\includegraphics[width=12cm]{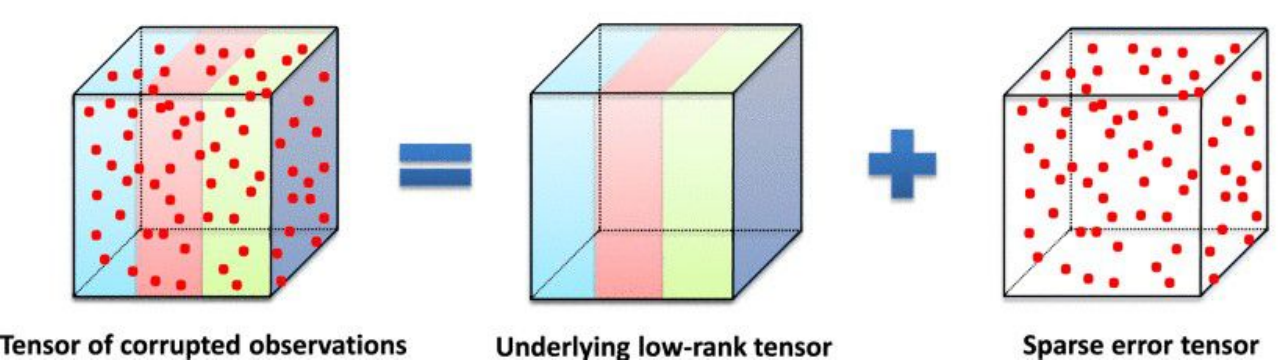}
\caption{Illustrations of RTPCA \cite{lu2016tensor}}
\end{figure}
robust tensor principal component analysis (RTPCA) \cite{lu2019tensor, lu2016tensor} aims to extract the low-rank and sparse tensors from the observed tensor as shown in FIGURE \ref{fig: RTPCA illstruation}, which has been applied to moving object tracking \cite{sobral2017robust}, image recovery \cite{zhou2017outlier}, and background modeling \cite{cao2016total}.  The tensor extension of RPCA is not easy since the tensor linear algebra is not well defined \cite{anandkumar2017homotopy}. One of the major issues is that it is difficult find the tight convex relaxation of the tensor rank. For example, computing the rank of CP decomposition \cite{kolda2009tensor} of a tensor is NP-hard \cite{hillar2013most}, leading to an intractable convex relaxation. 

Thanks to the development of tensor linear algebra,  the extension to RTPCA is possible. Specifically,  tensor-tensor product (t-product) \cite{kilmer2011factorization}, which is  a generalization of the matrix-matrix product, enjoys many similar properties to the matrix-matrix product. Based on the t-product, any tensors have the tensor singular value decomposition (t-SVD), which further motivates a new tensor rank, namely, tensor tubal rank \cite{kilmer2013third}. Under the framework of tensor linear algebra, the tensor nuclear norm and tensor spectral norm are defined in \cite{lu2019tensor} (see details in Section \ref{sec: notations and prelims}). Based on these definitions, RTPCA \cite{lu2019tensor, lu2016tensor} is developed, resulting the following problem: given that $\BFcalX=\BFcalL_0+\BFcalE_0 \in \mathbb{R}^{N_1\times N_2 \times N_3}$, where $\BFcalL_0\in \mathbb{R}^{N_1\times N_2 \times N_3}$ is the tensor having low tubal tensor rank and $\BFcalE_0 \in \mathbb{R}^{N_1\times N_2 \times N_3}$ is the sparse tensor, are we able to recover $(\BFcalL_0, \BFcalE_0)$  from the following convex optimization problem
\begin{align}\label{eq: convex RTPCA}
    \begin{aligned}
\underset{\BFcalL,\BFcalE}{\text{min}} \quad  &  \|\BFcalL\|_* + \gamma\|\BFcalE\|_1  \\
\text{subject to}  \quad  & \BFcalX =  \BFcalL + \BFcalE,
\end{aligned}
\end{align} 
where $\|\BFcalL\|_*$ is the tensor nuclear norm (see the definition in Section \ref{sec: notations and prelims}) to enforce the low-rank structure of $\BFcalL$, $\|\BFcalE\|_1$ is the $\ell_1$ norm to measure the sparsity of $\BFcalE$. 

The optimization \eqref{eq: convex RTPCA} can be solved by  a polynomial-time algorithm, namely, ADMM, with strong recovery guarantees \cite{lu2019tensor}. With the suggested parameter $\gamma=\frac{1}{\sqrt{\max\{N_1,N_2\}N_3}}$, the exact recovery is guaranteed with high probability under the assumption of \textit{tensor incoherence conditions} and a uniform model on the sparse support \cite{lu2019tensor}. Note that this result is the tensor extension of the main results in \citep{candes2011robust}. Unlike this conventional way, in this paper, the exact recovery is studied in a deterministic model, i.e., without assuming any randomness. Specifically, the contributions of this paper are summarized as follows: 
   \begin{itemize}
     \item An uncertainty principle between the low-rank tensor spaces and the sparsity pattern of a tensor is developed, which  characterizes fundamental identifiability.
     \item The deterministic sufficient condition (see Theorem \ref{thm: exact recovery 1}) for exact recovery is provided based on the optimization problem \eqref{eq: convex RTPCA}.
     \item Classes of low-rank and sparse tensors that satisfy the  deterministic sufficient condition in Theorem \ref{thm: exact recovery 1} are identified to guarantee the  exact recovery.
   \end{itemize}
\subsection{Related Work}
Except the RTPCA in \eqref{eq: convex RTPCA}, there are two other models of TRPCA using t-SVD \cite{zhou2017outlier,zhang2014novel}. 

Zhang \textit{et al}. \cite{zhang2014novel} proposes a RTPCA model in order to remove tubal noise, which results the following convex optimization problem
\begin{align}\label{eq: Zhang convex RTPCA}
    \begin{aligned}
\underset{\BFcalL,\BFcalE}{\text{min}} \quad  &  \|\BFcalL\|_{TNN} + \gamma\|\BFcalE\|_{1,1,2}  \\
\text{subject to}  \quad  & \BFcalX =  \BFcalL + \BFcalE,
\end{aligned}
\end{align} 
where $\|\BFcalL\|_{TNN}$ is another definition of tensor nuclear norm (see the definition in Section \ref{sec: notations and prelims}), $\|\BFcalE\|_{1,1,2}$ is defined as the sum of all the $\ell_2$ norms of the mode-3 (tube) fibers, i.e., $\sum_{n_1,n_2}\|\BFcalE(n_1,n_2,:)\|_2$. $\gamma=\frac{1}{\max{(N_1,N_2)}}$ is suggested for 3-way tensor $\BFcalX \in \mathbb{R}^{N_1\times N_2\times N_3}$, which is useful in practice. However, there is no exact recovery guarantee in this paper.

Zhou \textit{et al.} \cite{zhou2017outlier} proposes outlier-robust TPCA, which aims to deal with low-rank tensor observations with arbitrary outlier corruption. It has the following convex form:
\begin{align}\label{eq: OR-TPCA} 
    \begin{aligned}
\underset{\BFcalL,\BFcalE}{\text{min}} \quad  &  \|\BFcalL\|_{TNN} + \gamma\|\BFcalE\|_{2,1}  \\
\text{subject to}  \quad  & \BFcalX =  \BFcalL + \BFcalE,
\end{aligned}
\end{align} 
where $\|\BFcalE\|_{2,1}$ is the sum of Frobenius norms of lateral slices, i.e., $\|\BFcalE\|_{2,1}=\sum_{n_2=1}^{N_2}\|\BFcalE(:,n_2,:)\|_F$. With $\gamma=\frac{1}{\log(N_2)}$, the exact recovery is guaranteed with high probability under mild conditions. The models in \eqref{eq: convex RTPCA}, \eqref{eq: Zhang convex RTPCA} and \eqref{eq: OR-TPCA} are all convex optimization models, however, they have different sparse constrains designed for different applications.
\subsection{Organizations of the Paper}
The remaining of this paper is organized as follows. Section \ref{sec: notations and prelims} introduces the notations and preliminaries, which are basically related to the tensor algebraic framework in this paper. In Section \ref{sec: main results}, the deterministic results on RTPCA are presented, which are about the exact recovery under the t-product and t-SVD. The proofs are provided in Section \ref{sec: proofs} to support the main results. Finally, the implications of this work and future research are discussed in Section \ref{sec: conclusion}.

\section{Notations and Preliminaries} \label{sec: notations and prelims}
In this paper, matrices are represented by uppercase boldface letters, namely, $\BFA$;
vectors by lowercase boldface letters, namely, $\BFa$; and scalars are denoted by lowercase letters, namely, $a$. The boldface Euler script letters, e.g., $\BFcalA$, denote tensors. $\BFcalA(n_1,n_2,n_3)$ denotes the $(n_1,n_2,n_3)$-th entry of a 3-way tensor $\BFcalA\in\mathbb{R}^{N_1\times N_2\times N_3}$ and $\BFcalA(n_1,n_2,:)$ denotes the tube of tensor $\BFcalA\in\mathbb{R}^{N_1\times N_2\times N_3}$. The $n_1$-th horizontal, $n_2$-th lateral and $n_3$-th frontal slices are denoted as $\BFcalA(n_1, :, :)$, $\BFcalA(:, n_2 , :)$ and $\BFcalA(:, :, n_3)$, respectively. More often, the frontal slice $\BFcalA(:, :, n_3)$ is denoted compactly as $\BFA^{(n_3)}$.  The inner product between $\BFcalA$ and $\BFcalB$  in $\mathbb{R}^{N_1\times N_2 \times N_3}$ is defined as $\langle \BFcalA, \BFcalB \rangle=\sum_{n_3=1}^{N_3} \langle \BFA^{(n_3)}, \BFB^{(n_3)}\rangle$. $\|\BFcalA\|_F=\sqrt{\sum_{n_1,n_2,n_3}|\BFcalA(n_1,n_2,n_3)|^2}$, $\|\BFcalA\|_1=\sum_{n_1,n_2,n_3}|\BFcalA(n_1,n_2,n_3)|$ and $\|\BFcalA\|_{\infty}=\max_{n_1,n_2,n_3}|\BFcalA(n_1,n_2,n_3)|$ are the Frobenius norm, $\ell_1$ norm, and infinity norm of $\BFcalA$, respectively. 

$\Bar{\BFcalA}\in\mathbb{C}^{N_1\times N_2\times N_3}$ denotes the result of Discrete Fourier Transformation (DFT) \cite{golub2013matrix} on  $\BFcalA\in\mathbb{R}^{N_1\times N_2\times N_3}$ along the 3-rd dimension, i.e., performing the DFT on all the tubes of $\BFcalA$. $\Bar{\BFcalA}$ can be represented using the Matlab command $\normaltt{fft}$ as $\Bar{\BFcalA}=\normaltt{fft}(\BFcalA,[\;],3)$. Inversely, $\BFcalA$ can be computed through the inverse FFT, i.e., $\BFcalA=\normaltt{ifft}(\Bar{\BFcalA},[\;],3)$. 
Denote $\Bar{\BFA}\in \mathbb{C}^{N_1N_3\times N_2 N_3}$ as a block diagonal matrix with its $n_3$-th block on the diagonal as the $n_3$-th frontal slice $\Bar{\BFA}^{(n_3)}$ of $\Bar{\BFcalA}$, which has the following form
\begin{equation*} \label{eq: bdiag}
    \Bar{\BFA}=\normaltt{bdiag}(\Bar{\BFcalA})=\begin{bmatrix}
    \Bar{\BFA}^{(1)} & & & \\
       & \Bar{\BFA}^{(2)} &  & \\
  &  & \ddots & \\
  &  & & \Bar{\BFA}^{(N_3)}
  \end{bmatrix},
\end{equation*}
where operator $\normaltt{bdiag}$ maps the tensor $\Bar{\BFcalA}$ to the block diagonal matrix $\Bar{\BFA}$. Furthermore, the block circulant matrix $\normaltt{bcirc}(\BFcalA) \in \mathbb{R}^{N_1N_3\times N_2 N_3}$ of $\BFcalA$  is defined as 
\begin{equation} \label{eq: bcirc}
   \normaltt{bcirc}(\BFcalA)=\begin{bmatrix}
   \BFA^{(1)}    & \BFA^{(N_3) }    & \cdots &   \BFA^{(2)}  \\
   \BFA^{(2)}    & \BFA^{(1)}       & \cdots &    \BFA^{(3)}  \\
 \vdots          & \vdots           & \ddots &      \vdots          \\
\BFA^{(N_3)}     & \BFA^{(N_3-1)}   & \cdots &   \BFA^{(1)}  
  \end{bmatrix}.
\end{equation}

Next, the framework of t-product and t-SVD is introduced. For $\BFcalA\in\mathbb{R}^{N_1\times N_2\times N_3}$, define
\begin{equation*}
    \normaltt{unfold}(\BFcalA)=\begin{bmatrix}
   \BFA^{(1)}     \\
   \BFA^{(2)}    \\
 \vdots          \\
\BFA^{(N_3)}  
  \end{bmatrix}, \; \normaltt{fold}(\normaltt{unfold}(\BFcalA))=\BFcalA,
\end{equation*}
the matricization and  tensorization operators \cite{kilmer2011factorization}, respectively.

\begin{definition}[T-product \cite{kilmer2011factorization}] Let $\BFcalA\in\mathbb{R}^{N_1\times N_2\times N_3}$ and  $\BFcalB\in\mathbb{R}^{N_2 \times L\times N_3}$. Then the t-product $\BFcalA * \BFcalB$ is defined as,
\begin{equation*} \label{eq: t-product}
    \BFcalA * \BFcalB = \normaltt{fold}(\normaltt{bcirc}(\BFcalA)\cdot \normaltt{unfold}(\BFcalB)) \in \mathbb{R}^{N_1\times L\times N_3}.
\end{equation*}
\end{definition}
The  t-product can be understood as follows. A 3-way tensor of size $N_1\times N_2 \times N_3$ can be regarded as an $N_1\times N_2$ matrix with each entry being a tube that lies in the third dimension. Thus, the t-product is analogous to the matrix multiplication except
that the circular convolution replaces the multiplication operation between the elements.
\begin{definition}[Tensor transpose \cite{kilmer2011factorization}] The tensor transpose of a tensor $\BFcalA\in\mathbb{R}^{N_1\times N_2\times N_3}$ is the tensor $\BFcalA^{\top}\in\mathbb{R}^{N_1\times N_2\times N_3}$ obtained by transposing each of the frontal slices and then reversing the order of transposed frontal slices 2 through $N_3$.
\end{definition}
As an example, let $\BFcalA\in\mathbb{R}^{N_1\times N_2\times 4}$, then 
\begin{equation*}
    \BFcalA^{\top} =\normaltt{fold}	\left(
    \begin{bmatrix}
   \BFA^{(1)\top}    \\
   \BFA^{(4)\top}    \\
   \BFA^{(3)\top}    \\
   \BFA^{(2)\top}  
  \end{bmatrix}\right).
\end{equation*}
\begin{definition}[Identity tensor \cite{kilmer2011factorization}] The identity tensor $\BFcalI\in\mathbb{R}^{N\times N \times N_3}$ is the tensor with its first frontal slice being the
$N\times N$ identity matrix, and other frontal slices being all zeros.
\end{definition}
\begin{definition}[Orthogonal tensor \cite{kilmer2011factorization}] A tensor $\BFcalT\in\mathbb{R}^{N\times N \times N_3}$ is orthogonal if it satisfies $\BFcalT^{\top} * \BFcalT = \BFcalT * \BFcalT^{\top} = \BFcalI$.
\end{definition}

\begin{definition}[F-diagonal Tensor \cite{kilmer2011factorization}] A tensor is called f-diagonal if each of its frontal slices is a diagonal matrix.
\end{definition}
\begin{figure}[!htbp] \label{fig: T-SVD}
\centering
\includegraphics[width=12cm]{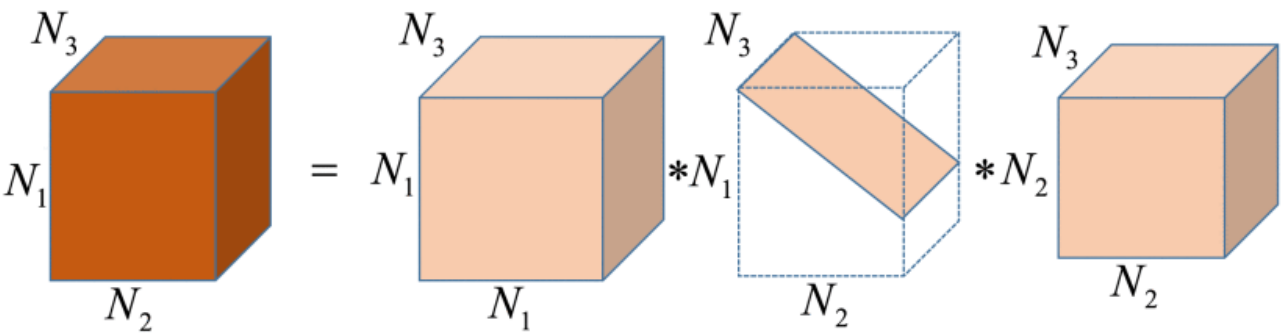}
\caption{Illustration of t-SVD for an $N_1 \times N_2 \times N_3$ tensor \cite{liu2018improved}}
\end{figure}

\begin{theorem}[T-SVD \cite{kilmer2011factorization}] \label{thm: t-SVD}
Let $\BFcalA\in\mathbb{R}^{N_1\times N_2\times N_3}$.  Then it can factorized as 
\begin{equation}\label{eq: T-SVD}
    \BFcalA=\BFcalU * \BFcalS * \BFcalV^{\top},
\end{equation}
where  $\BFcalU \in\mathbb{R}^{N_1\times N_1\times N_3}, \BFcalV \in\mathbb{R}^{N_2\times N_2\times N_3}$ are orthogonal, and $\BFcalS\in\mathbb{R}^{N_1\times N_2\times N_3}$ is an f-diagonal tensor.
\end{theorem}
The above Theorem \ref{thm: t-SVD} shows that any 3-way tensor can be factorized into three components, namely, two orthogonal tensors and an f-diagonal tensor. The FIGURE \ref{fig: T-SVD} shows an illustration of the t-SVD factorization. An efficient way for computing t-SVD is implemented in \cite{lu2018tensor}. 

\begin{definition}[Tensor tubal rank \cite{kilmer2013third, zhang2014novel}] For a tensor $\BFcalA\in\mathbb{R}^{N_1\times N_2\times N_3}$, the tensor tubal rank, denoted as $\normaltt{rank}_t(\BFcalA)$, is defined as the number of nonzero singular tubes of $\BFcalS$, where $\BFcalS$ is from the t-SVD of $\BFcalA=\BFcalU * \BFcalS * \BFcalV^{\top}$. Then,
\begin{equation*}
    \normaltt{rank}_t(\BFcalA) =\sharp\{i,\BFcalS(i,i,:)\neq \bm{0}\}.
\end{equation*}
The tensor tubal rank can be further determined by the first slice $\BFcalS(:,:,1)$ of $\BFcalS$ \cite{lu2019tensor}, i.e.,
\begin{equation*}
    \normaltt{rank}_t(\BFcalA) =\sharp\{i,\BFcalS(i,i,1)\neq 0\}.
\end{equation*}
\end{definition}

The entries on the diagonal of the first frontal slice $\BFcalS(:,:,1)$ of $\BFcalS$ have the same decreasing property as matrix, namely,
\begin{equation*}
    \BFcalS(1,1,1)\geq \cdots \geq  \BFcalS(R,R,1) \geq 0.
\end{equation*}
Hence, the tensor tubal rank is equivalent to the number of nonzero singular values of $\BFcalA$.
\begin{remark}
 It is usually sufficient to compute the reduced version of the t-SVD using the tensor tubal-rank. It’s faster and more economical for storage. In details, suppose $\BFcalA\in\mathbb{R}^{N_1\times N_2\times N_3}$ has tensor tubal-rank $R$, then the skinny t-SVD of $\BFcalA$ is given by
 \begin{equation*}\label{eq: skinnt T-SVD}
    \BFcalA=\BFcalU * \BFcalS * \BFcalV^{\top},
\end{equation*}
where  $\BFcalU \in\mathbb{R}^{N_1\times R \times N_3}, \BFcalV \in\mathbb{R}^{N_2\times R \times N_3}$ satisfying $\BFcalU^{\top}\BFcalU=\BFcalI, \BFcalV^{\top}\BFcalV=\BFcalI$, and $\BFcalS\in\mathbb{R}^{R\times R\times N_3}$ is an f-diagonal tensor.
\end{remark}
\textbf{The skinny t-SVD will be used throughout the paper unless otherwise noted.}
\begin{definition}[Tensor spectral norm \cite{lu2019tensor}] The tensor spectral norm of $\BFcalA\in\mathbb{R}^{N_1\times N_2\times N_3}$ is defined as $\|\BFcalA\| \coloneqq \|\normaltt{bcirc}(\BFcalA)\|$.
\end{definition}
The tensor nuclear norm, denoted as $\|\BFcalA\|_*$, as the dual norm of the tensor spectral norm is defined as follows.
\begin{definition}[Tensor nuclear norm \cite{lu2019tensor}] Let $\BFcalA=\BFcalU * \BFcalS * \BFcalV^{\top}$ be t-SVD of $\BFcalA\in\mathbb{R}^{N_1\times N_2\times N_3}$. The tensor nuclear norm of $\BFcalA$ is defined as
\begin{equation*}
    \|\BFcalA\|_* \coloneqq \langle\BFcalS, \BFcalI \rangle = \sum_{i=1}^R\BFcalS(i,i,1), 
\end{equation*}
where $R=\normaltt{rank}_t(\BFcalA)$. Furthermore, the tensor nuclear norm can be rewritten as \
\begin{equation*}
    \|\BFcalA\|_* \coloneqq \frac{1}{N_3} \|\normaltt{bcirc}(\BFcalA)\|_*.
\end{equation*}
\end{definition}
Note that there is another tensor nuclear norm, namely, $\|\BFcalA\|_{TNN}$, defined as $\sum_{n_3=1}^{N_3}\sum_{i=1}^R\BFcalS(i,i,n_3)$  \cite{zhang2014novel,zhou2017outlier}. Based on t-SVD \eqref{eq: T-SVD}, the subgradient of tensor nuclear norm, which is very important to the proofs in Section \ref{sec: proofs}, is defined as follow.
\begin{theorem}[Subgradient of tensor nuclear norm]
Let $\BFcalA\in\mathbb{R}^{N_1\times N_2\times N_3}$  with $\normaltt{rank}_t(\BFcalA) = R$ and its skinny t-SVD be $\BFcalA=\BFcalU * \BFcalS * \BFcalV^{\top}$. The subdifferential (the set of subgradients) of   $\|\BFcalA\|_*$ is $\partial\|\BFcalA\|_*=\{\BFcalU *\BFcalV^{\top} + \BFcalW| \BFcalU^{\top} * \BFcalW=\bm{0}, \BFcalW * \BFcalV = \bm{0},\|\BFcalW\|\leq 1 \}$.
\end{theorem}
\begin{definition}[Standard tensor basis \cite{zhang2016exact}]
Denote $\mathring{\mathpzc{e}}_n$ as the tensor column basis, which is a tensor of $N\times 1 \times N_3$ with its $(n,1,1)$-th entry equal to $1$ and the rest equal to $0$.  Naturally its transpose $\mathring{\mathpzc{e}}_n^{\top}$  is called row basis. $\dot{\mathpzc{e}}_k$ denotes the tensor tube basis, which is a tensor of size $1\times 1 \times N_3$ with its $(1,1,k)$-th  entry
equaling $1$ and the rest equaling $0$.
\end{definition}
\section{Main Results} \label{sec: main results}
Throughout this paper, the analysis focuses on 3-way tensors. The analysis can be extended to $p$-way ($p> 3$) tensors using the t-SVD for $p$-way tensors defined in \cite{martin2013order}. Given that $\BFcalX=\BFcalL_0+\BFcalE_0$, where $\BFcalL_0$ is an unknown low tubal rank tensor and $\BFcalE_0$ is an unknown sparse tensor, the exact recovery $\BFcalL_0,\BFcalE_0$ from $\BFcalX$ without any additional condition is impossible.  In the remaining of this section, two natural problems of identifiability are introduced. Based on that, the sufficient condition for the exact recovery is provided. At the end, the classes of low-rank tensor spaces and sparse tensors are identified for the exact recovery.
\subsection{Deterministic Exact Recovery}
This section is dedicated to presenting the exact recovery results under a  deterministic model. We begin with two natural identifiability problems introduced in \cite{chandrasekaran2011rank}: (1) the low-rank tensor is very sparse itself; (2) the sparse tensor has all its support concentrated in one horizontal/lateral slice. To deal with these two problems, two concepts are introduced as follows.

\textit{(1) The low-rank tensor is very sparse itself.} It can be addressed by imposing certain conditions on the tensor spaces $\BFcalU,\BFcalV$ of the low-rank tensor $\BFcalA$. For any tensor $\BFcalA$, the tangent space $T(\BFcalA)$ at $\BFcalA$ with respect to the variety of all tensors with tubal rank less than or equal to $\normaltt{rank}_t(\BFcalA)=R$ is the span of tensor spaces $\BFcalU \in \mathbb{R}^{N_1\times R\times N_3},\BFcalV \in \mathbb{R}^{N_2\times R\times N_3}$, where $\BFcalA = \BFcalU * \BFcalS * \BFcalV^{\top}$ is the t-SVD. The $T(\BFcalA)$ can be represented as
\begin{equation} \label{eq: tangent space LR tensor}
    T(\BFcalA)=\{\BFcalU * \BFcalY^{\top} + \BFcalW * \BFcalV^{\top},\BFcalY\in \mathbb{R}^{N_2\times R \times N_3},\BFcalW \in \mathbb{R}^{N_1\times R \times N_3}\}.
\end{equation}
It always has that $\BFcalA\in T(\BFcalA)$. Next, one of the key definitions, namely, $\xi(\BFcalA)$, in this paper, which measures the ``sparsity'' of the contraction (tensor nuclear norm less than and equal $1$) in the tangent space as defined in  \eqref{eq: tangent space LR tensor}, is defined  as follows:
\begin{equation} \label{eq: low rank condition}
    \xi(\BFcalA) = \underset{\BFcalN \in T(\BFcalA),\|\BFcalN\|\leq 1}{\text{max}} \|\BFcalN\|_{\infty},
\end{equation}
where $\|\cdot\|$ and $\|\cdot\|_{\infty}$ are the tensor spectral norm and infinity norm, respectively. If $\xi(\BFcalA)$ is small, it implies that the tensors in the tangent space $T(\BFcalA)$ is not very sparse.

\textit{(2) the sparse tensor has all its support concentrated in one horizontal/lateral slice.} If the sparse tensor has all its support concentrated in one horizontal/lateral slice; the entries in this horizontal/lateral slice could negate the entries of the corresponding low-rank tensor, thus leaving the tensor tubal rank and the tensor spaces of the low-rank tensor unchanged. In order to address this problem, some conditions should be imposed on the sparsity pattern of the sparse tensor such that the support of tensor is not too concentrated in any horizontal/lateral slice. For any tensor $\BFcalA$, the tangent space $\Omega(\BFcalA)$ with respect to $\normaltt{support}(\BFcalA)$ at $\BFcalA$ is defined as follow:
\begin{equation} \label{eq: sparse tangent space}
    \Omega(\BFcalA) =\{\BFcalN\in \mathbb{R}^{N_1\times N_2\times N_3}|\normaltt{support}(\BFcalN) \subseteq \normaltt{support}(\BFcalA)\}.
\end{equation}
Furthermore, $\mu(\BFcalA)$ is defined to measure  the ``concentration'' of the tensor in tangent space  $\Omega(\BFcalA)$ \eqref{eq: sparse tangent space} as
\begin{equation} \label{eq: sparse condition}
    \mu(\BFcalA) = \underset{\BFcalN \in \Omega(\BFcalA),\|\BFcalN\|_{\infty}\leq 1}{\text{max}} \|\BFcalN\|.
\end{equation}
If $\mu(\BFcalA)$ is small, it implies that the support for the tensor in the tangent space $\Omega(\BFcalA)$ is not too concentrated in any horizontal/lateral slice.

Before getting into the exact recovery of $(\BFcalL_0,\BFcalE_0)$ from the optimization \eqref{eq: convex RTPCA}, one question deserves further discussion. That is, will $(\BFcalL_0,\BFcalE_0)$ be uniquely recovered from its observation $\BFcalX$? Because the convex optimization may have multiple optimal solutions, however, there is only one  $(\BFcalL_0,\BFcalE_0)$ need to be  recovered, which leads to a mismatching situation. In order to avoid this situation,  $(\BFcalL_0,\BFcalE_0)$ should be the unique minimizer of \eqref{eq: convex RTPCA}. Here, the necessary and sufficient condition of unique recovery with respect to the tangent spaces $T(\BFcalL_0)$ and $\Omega(\BFcalE_0)$ is identified, which can be summarized in the following lemma.
\begin{lemma} \label{lemma: unique decomposition}
Assuming that $\BFcalL\in T(\BFcalL_0)$, $\BFcalE \in \Omega(\BFcalE_0)$, and $\BFcalL+\BFcalE=\BFcalX$, then $(\BFcalL,\BFcalE)$ has to be $(\BFcalL_0,\BFcalE_0)$ if and only if 
\begin{equation} \label{eq: trivial intersection}
    T(\BFcalL_0) \cap \Omega(\BFcalE_0) = \{\bm{0}\}.
\end{equation}
\end{lemma}
This lemma tell us that if the $T(\BFcalL_0)$ and $\Omega(\BFcalE_0)$ have a trivial intersection, then we will have the unique decomposition. Furthermore,  based on the inverse function theorem, the condition in \ref{lemma: unique decomposition} is also   a sufficient condition for local identifiability around $(\BFcalL_0,\BFcalE_0)$ with respect to the low-rank and sparse tensor varieties. Next, a sufficient condition is provided in terms of numerical values $\xi(\BFcalL_0)$ and $\mu(\BFcalE_0)$ to guarantee a trivial intersection between $T(\BFcalL_0)$ and $\Omega(\BFcalE_0)$.
\begin{lemma}\label{lemma: sufficient for trivial intersection}
Given any two tensors $\BFcalL_0$ and $\BFcalE_0$, we have that
\begin{equation*}
    \xi(\BFcalL_0)\mu(\BFcalE_0)<1 \Rightarrow  T(\BFcalL_0) \cap \Omega(\BFcalE_0) = \{\bm{0}\},
\end{equation*}
where $\xi(\BFcalL_0),\mu(\BFcalE_0)$ are defined as \eqref{eq: low rank condition} and \eqref{eq: sparse condition}, respectively.
\end{lemma}
This lemma gives a numerical sufficient condition, namely, $\xi(\BFcalL_0)\mu(\BFcalE_0)<1$, for an algebraic conclusion, namely, $ T(\BFcalL_0) \cap \Omega(\BFcalE_0) = \{\bm{0}\}$ in \eqref{eq: trivial intersection}. Specifically, both $\xi(\BFcalL_0)$ and $\mu(\BFcalE_0)$ being small implies that the tangent spaces $T(\BFcalL_0)$ and $\Omega(\BFcalE_0)$ intersect transversally. According to lemma \ref{lemma: sufficient for trivial intersection}, the tensor rank-sparsity uncertainty principle can be obtained easily as a corollary since $\BFcalA\in T(\BFcalA) \cap \Omega(\BFcalA)$. That is, there is no tensor, which can be too sparse while having “diffuse” tensor spaces.
\begin{corollary} \label{corollary: rank-sparsity uncertainty}
For any tensor $\BFcalA\neq \bm{0}$, the following holds
$$ \xi(\BFcalA)\mu(\BFcalA)\geq 1,$$
where $\xi(\BFcalA),\mu(\BFcalA)$ are defined as \eqref{eq: low rank condition} and \eqref{eq: sparse condition}, respectively.
\end{corollary}
 In other words, for any tensor $\BFcalA\neq \bm{0}$ both $\xi(\BFcalA)$ and $\mu(\BFcalA)$ cannot be simultaneously small.  The key result of this paper is presented in the following theorem.
\begin{theorem}\label{thm: exact recovery 1}
Given that $\BFcalX=\BFcalL_0+\BFcalE_0$ with 
$$\xi(\BFcalL_0)\mu(\BFcalE_0) < \frac{1}{6},$$
then $(\BFcalL_0, \BFcalE_0)$ is the unique minimizer to \eqref{eq: convex RTPCA} for the following range of $\gamma$:
$$ \gamma \in \left( \frac{\xi(\BFcalL_0)}{1-4\xi(\BFcalL_0)\mu(\BFcalE_0)}, \frac{1-3\xi(\BFcalL_0)\mu(\BFcalE_0)}{\mu(\BFcalE_0)} \right).$$
Specifically, $\gamma=\frac{(3\xi(\BFcalL_0))^p}{(2\mu(\BFcalE_0))^{1-p}}$ for any choice of $p\in [0,1]$ is always inside the above range and thus guarantees exact recovery of $(\BFcalL_0, \BFcalE_0)$. 
\end{theorem}
The above result shows that if $\xi(\BFcalL_0)\mu(\BFcalE_0) < \frac{1}{6}$, the exact recovery of $(\BFcalL_0, \BFcalE_0)$ is guaranteed. Note that $\xi(\BFcalL_0)\mu(\BFcalE_0) < \frac{1}{6}$ guarantees that the tangent spaces $T(\BFcalL_0)$ and $\Omega(\BFcalE_0)$ are sufficiently transverse based on Lemma \ref{lemma: sufficient for trivial intersection}. 
\subsection{Characterization of Low-Rank and Sparse Tensors} \label{subsec: classes of low-rank and sparse tensors}
In this section, the classes of low-rank and sparse tensors that satisfy the sufficient condition in Theorem \ref{thm: exact recovery 1}. We begin with the low-rank tensor with small $\xi$. Specifically, we show that
tensor with tensor spaces that are incoherent with respect to the standard tensor basis have small $\xi$. The tensor incoherence of a tensor subspace $\BFcalS$ can be measured as follows:
\begin{equation} \label{eq: tensor incoherence}
\beta(\BFcalS)\coloneqq \max_{n}\|\BFcalP_{\BFcalS}\mathring{\mathpzc{e}}_n\|_F,
\end{equation}
where $\BFcalP_{\BFcalS}$ is the projection onto the tensor subspace $\BFcalS$. The above definition \eqref{eq: tensor incoherence} is the same definition in the \textit{Tensor Incoherence Conditions} used in \cite{zhou2017outlier,lu2019tensor, zhang2016exact, liu2018improved}. A small value of $\beta(\BFcalS)$ implies that the tensor subspace $\BFcalS$ is not closely aligned with any of the coordinate axes. Given t-SVD $\BFcalL=\BFcalU * \BFcalS * \BFcalV^{\top}$, the incoherence of the tensor spaces of a tensor $\BFcalL \in \mathbb{R}^{N_1\times N_2 \times N_3}$ can be defined as
\begin{equation} \label{eq: tensor incoherence max}
    \normaltt{inc}(\BFcalL) \coloneqq \max\{\beta(\normaltt{span}(\BFcalU)), \beta(\normaltt{span}(\BFcalV))\},
\end{equation}
 where $\normaltt{span}(\BFcalU)$ and $\normaltt{span}(\BFcalV)$ are the smallest tensor linear space contains $\BFcalU$ and $\BFcalV$, respectively. One question is that: what is the relationship between $\xi(\BFcalL)$ and $\normaltt{inc}(\BFcalL)$? Because  $\xi(\BFcalL)$ is also related to $\normaltt{inc}(\BFcalL)$. The following lemma  describes the numerical relationship between $\xi(\BFcalL)$ and $\normaltt{inc}(\BFcalL)$. That is,  $\xi(\BFcalL)$ is lower and upper bounded by the tensor incoherence $\normaltt{inc}(\BFcalL)/\sqrt{N_3}$ and $\normaltt{inc}(\BFcalL)$, respectively.
\begin{lemma} \label{lemma: low-rank tensor property}
Let $\BFcalL\in \mathbb{R}^{N_1\times N_2 \times N_3}$ be any tensor with  $\normaltt{inc}(\BFcalL)$ defined in \eqref{eq: tensor incoherence max}.  Then the following result holds:
\[\normaltt{inc}(\BFcalL)/\sqrt{N_3} \leq \xi(\BFcalL) \leq 2\normaltt{inc}(\BFcalL).\]
\end{lemma}
Note that the lower bound is different from the result in \cite{chandrasekaran2011rank} since we deal with the case of 3-way tensors. If $N_3=1$, then Lemma \ref{lemma: low-rank tensor property} reduces to  result in \cite{chandrasekaran2011rank}.

On the other hand, it is also important to know what kind of sparse tensor has small $\mu$. We find that the $\mu$ of a tensor is upper bounded by the maximum number of nonzero entries per horizontal/lateral slice, lower bounded by the minimum number of nonzero entries per horizontal/lateral slice as stated in the below lemma.
\begin{lemma} \label{lemma: sparsity pattern}
Let $\BFcalE\in \mathbb{R}^{N_1\times N_2 \times N_3}$ be any tensor with at most $\normaltt{deg}_{max}(\BFcalE)$ nonzero entries per horizontal/lateral slice and with  at least $\normaltt{deg}_{min}(\BFcalE)$ nonzero entries per horizontal/lateral slice. The following holds
\[\normaltt{deg}_{min}(\BFcalE) \leq \mu(\BFcalE) \leq \normaltt{deg}_{max}(\BFcalE).\]
\end{lemma}
This lemma provides lower and upper bounds of $\mu$ using $\normaltt{deg}_{min}(\BFcalE)$ and $\normaltt{deg}_{max}(\BFcalE)$, respectively. However, the upper bound can not exactly characterize the sparsity pattern of a tensor, which is essential to determine the value of $\mu$. For example, a 3-way tensor $\BFcalE_1\in \mathbb{R}^{N_1\times N_2 \times N_3}$ has one tube of 1 and all reaming tubes of 0, then $\normaltt{deg}_{max}(\BFcalE)=N_3$, which has the same value of a tensor $\BFcalE_2\in \mathbb{R}^{N_1\times N_2 \times N_3}$ with 1 everywhere.

Taking advantages of Lemma \ref{lemma: low-rank tensor property} and \ref{lemma: sparsity pattern} together with Theorem \ref{thm: exact recovery 1}, the following corollary can be concluded. That is, a small product of the low-rank tensor incoherence  and bounded sparse tensor implies the exact recovery from the convex optimization \eqref{eq: convex RTPCA}.
\begin{corollary} \label{corollary: exacty recovery}
Given that $\BFcalX=\BFcalL_0+\BFcalE_0$ with  $\normaltt{inc}(\BFcalL_0)$ and $\normaltt{deg}_{max}(\BFcalE_0)$, if we have
$$\normaltt{inc}(\BFcalL_0)\normaltt{deg}_{max}(\BFcalE_0)<\frac{1}{12},$$
then $(\BFcalL_0, \BFcalE_0)$ is the unique minimizer to \eqref{eq: convex RTPCA} for the following range of $\gamma$:
$$ \gamma \in \left( \frac{2\normaltt{inc}(\BFcalL_0)}{1-8\normaltt{inc}(\BFcalL_0)\normaltt{deg}_{max}(\BFcalE_0)}, \frac{1-\normaltt{inc}(\BFcalL_0)\normaltt{deg}_{max}(\BFcalE_0)}{\normaltt{deg}_{max}(\BFcalE_0)} \right).$$
Specifically, $\gamma=\frac{(6\normaltt{inc}(\BFcalL_0))^p}{(2\normaltt{deg}_{max}(\BFcalE_0))^{1-p}}$ for any choice of $p\in [0,1]$ is always inside the above range and thus guarantees exact recovery of $(\BFcalL_0, \BFcalE_0)$.
\end{corollary}

\section{Proofs for Main Results} \label{sec: proofs}
This section introduces the key steps underlying the proofs related to the main results in Section \ref{sec: main results}. The notations related to the proofs in this section are introduced first. The orthogonal projection onto the space $T(\BFcalL_0)$ is denoted as $\BFcalP_{T(\BFcalL_0)}$. Given   t-SVD of $\BFcalL_0 =\BFcalU* \BFcalS * \BFcalV$,   $\BFcalP_{T(\BFcalL_0)}$ has the following explicit form:
\begin{equation*} \label{eq: projection to the low-rank space}
    \BFcalP_{T(\BFcalL_0)}(\BFcalA)=\BFcalP_{\BFcalU}*\BFcalA + \BFcalA*\BFcalP_{\BFcalV}  - \BFcalP_{\BFcalU}*\BFcalA*\BFcalP_{\BFcalV},
\end{equation*}
where $\BFcalP_{\BFcalU}=\BFcalU * \BFcalU^{\top}$ and  $\BFcalP_{\BFcalV}=\BFcalV * \BFcalV^{\top}$. Denote the orthogonal  space to $T(\BFcalL_0)$ as $T(\BFcalL_0)^{\perp}$.   The orthogonal projection onto the space $T(\BFcalL_0)^{\perp}$ is denoted as $\BFcalP_{T(\BFcalL_0)^{\perp}}$, which has the following form
\begin{equation*} \label{eq: projection to the orthogonal low-rank space}
    \BFcalP_{T(\BFcalL_0)^{\perp}}(\BFcalA)=(\BFcalI_{N_1} - \BFcalP_{\BFcalU})*\BFcalA *(\BFcalI_{N_2} - \BFcalP_{\BFcalV}) ,
\end{equation*}
 where $\BFcalI_{N_1}$ and $\BFcalI_{N_2}$ are $N_1\times N_1 \times N_3$ and $N_2 \times N_2 \times N_3$ identity tensors, respectively.

Similarly, the orthogonal projection onto the space $\Omega(\BFcalE_0)$ is denoted as $\BFcalP_{\Omega(\BFcalE_0)}$, which simply sets to zero those entries with support not inside $\normaltt{support}(\BFcalE_0)$. The orthogonal projection onto the space $\Omega(\BFcalE_0)$ is denoted as $\Omega(\BFcalE_0)^{\complement}$, which consists of tensors with complementary support, i.e., $\normaltt{support}(\BFcalE_0)^{\complement}$. The projection onto $\Omega(\BFcalE_0)^{\complement}$ is denoted as $\BFcalP_{\Omega(\BFcalE_0)^{\complement}}$.

\subsection{Proofs of Lemmas \ref{lemma: unique decomposition} and \ref{lemma: sufficient for trivial intersection}} In this subsection, we provide the proofs related to the tangent spaces $ T(\BFcalL)$ and $\Omega(\BFcalE)$.
\begin{proof}[Proof of Lemma \ref{lemma: unique decomposition}] 
\textit{Sufficient condition:} It is easily seen by observation. \newline
\textit{Necessary condition:}  This part can be proved by contradiction. Assume that there is a nonzero tensor $\BFcalB$ such that 
\[ \BFcalB \in T(\BFcalL_0) \cap \Omega(\BFcalE_0). \]
$\BFcalL_1 = \BFcalL_0 - \BFcalB$ and $\BFcalE_1 = \BFcalE_0 + \BFcalB$ satisfy $\BFcalL_1 \in T(\BFcalL_0)$, $\BFcalE_1 \in \Omega(\BFcalE_0)$ and $\BFcalL_1 + \BFcalE_1 =\BFcalX$, which is a contradiction.
\end{proof}
\begin{proof}[Proof of Lemma \ref{lemma: sufficient for trivial intersection}]
 First, the following statement is established 
\begin{equation} \label{eq: proof of lemma 2-1}
    \underset{\BFcalN \in T(\BFcalL_0),\|\BFcalN\|\leq 1}{\text{max}} \|\BFcalP_{\Omega(\BFcalE_0)}(\BFcalN)\| < 1\Rightarrow  T(\BFcalL_0) \cap \Omega(\BFcalE_0) = \{\bm{0}\}.
\end{equation}
The above statement can be proved by contradiction. Assume that the above statement is not true. In other words, there exist $\BFcalN \neq 0$ such that $\BFcalN \in T(\BFcalL_0) \cap \Omega(\BFcalE_0)$,  $\BFcalN \in T(\BFcalL_0)$ and   $\|\BFcalN\|\leq 1$. $\BFcalN$ can be appropriately scaled such that $\|\BFcalN\|=1$. However, $\|\BFcalP_{\Omega(\BFcalE_0)}(\BFcalN)\|=\|\BFcalN\|=1$ since $\BFcalN \in T(\BFcalL_0) \cap \Omega(\BFcalE_0)$, which leads to a contradiction. Next,
\begin{align} \label{eq: proof of lemma 2-2}
    \begin{aligned}
  \underset{\BFcalN \in T(\BFcalL_0),\|\BFcalN\|\leq 1}{\text{max}} \|\BFcalP_{\Omega(\BFcalE_0)}(\BFcalN)\| & \leq  \underset{\BFcalN \in T(\BFcalL_0),\|\BFcalN\|\leq 1}{\text{max}} \mu(\BFcalE_0)\|\BFcalP_{\Omega(\BFcalE_0)}(\BFcalN)\|_{\infty}  \\
 & \leq  \underset{\BFcalN \in T(\BFcalL_0),\|\BFcalN\|\leq 1}{\text{max}} \mu(\BFcalE_0)\|\BFcalN\|_{\infty}  \\
 & \leq  \mu(\BFcalE_0) \xi(\BFcalL_0) \\
 & < 1,
\end{aligned}    
\end{align}
where the first inequality follows from the definition of $\mu(\BFcalE_0)$ since $\BFcalP_{\Omega(\BFcalE_0)}(\BFcalN)\in\Omega(\BFcalE_0)$,  the second inequality is due to the fact that $\|\BFcalP_{\Omega(\BFcalE_0)}(\BFcalN)\|_{\infty}\leq \|\BFcalN\|_{\infty}$, the third inequality is based on the definition of $\xi(\BFcalE_0)$, the last inequality is the condition given in Lemma \ref{lemma: sufficient for trivial intersection}. According to \eqref{eq: proof of lemma 2-1} and \eqref{eq: proof of lemma 2-2}, the proof is concluded.
\end{proof}
\subsection{Proof of Theorem \ref{thm: exact recovery 1}} In this subsection, the proof for Theorem \ref{thm: exact recovery 1} is provided, which consists of the following two steps: (1) Sufficient conditions for exact recovery are provided in Lemma \ref{lemma: sufficient condition for exact recovery}; (2) The conditions in Theorem \ref{thm: exact recovery 1} satisfy the sufficient conditions given in Lemma \ref{lemma: sufficient condition for exact recovery}. We begin with  stating the Lemma \ref{lemma: sufficient condition for exact recovery}, and then the corresponding proof is provided accordingly.
\begin{lemma} \label{lemma: sufficient condition for exact recovery}
Given that $\BFcalX=\BFcalL_0 + \BFcalE_0$. Then $(\BFcalL_0, \BFcalE_0)$ is the unique minimizer of  \eqref{eq: convex RTPCA} if the following conditions are satisfied:
 \begin{enumerate}
   \item $T(\BFcalL_0) \cap \Omega(\BFcalE_0)=\{\bm{0}\}$.
   \item There exists a dual $\BFcalQ \in \mathbb{R}^{N_1 \times N_2 \times N_3}$ such that
   \begin{enumerate}
     \item $\BFcalP_{T(\BFcalL_0)}(\BFcalQ)=\BFcalU * \BFcalV^{\top}$,
     \item $\|\BFcalP_{T(\BFcalL_0)^{\perp}}(\BFcalQ)\|< 1$,
     \item $\BFcalP_{\Omega(\BFcalE_0)}(\BFcalQ)=\gamma \normaltt{sign}(\BFcalE_0)$,
     \item $\|\BFcalP_{\Omega(\BFcalE_0)^{\complement}}(\BFcalQ)\|_{\infty}<\gamma$,
   \end{enumerate}
 \end{enumerate}
 where $\normaltt{sign}(\BFcalE_0(n_1,n_2,n_3))$ equals  $+1$ if $\BFcalE_0(n_1,n_2,n_3)>0$, $-1$  if $\BFcalE_0(n_1,n_2,n_3)<0$, and $0$ if  $\BFcalE_0(n_1,n_2,n_3)=0$.
 \end{lemma}
 \begin{proof}[Proof of Lemma \ref{lemma: sufficient condition for exact recovery}]
We will show that $(\BFcalL_0, \BFcalE_0)$ is the  minimizer first. From the optimality conditions for a convex optimization problem \cite{boyd2004convex}, $(\BFcalL_0, \BFcalE_0)$ is a minimizer if and only if there exists a dual $\BFcalQ\in \mathbb{R}^{N_1 \times N_2 \times N_3}$ such that 
\begin{equation*}
    \BFcalQ \in \partial\|\BFcalL_0\|_* \textnormal{ and }\BFcalQ \in \gamma \partial\|\BFcalE_0\|_1.
\end{equation*}
Here $\BFcalQ \in \partial\|\BFcalL_0\|_*$ if and only if \cite{watson1992characterization}:
\begin{equation} \label{eq: proof of lemma 5-1}
    \BFcalP_{T(\BFcalL_0)}(\BFcalQ)=\BFcalU * \BFcalV^{\top} \textnormal{ and } \|\BFcalP_{T(\BFcalL_0)^{\perp}}(\BFcalQ)\| \leq 1.
\end{equation}
Based on the properties of the subdifferential of the $\ell_1$ norm, $\BFcalQ \in \gamma \partial\|\BFcalE_0\|_1$ if and only if 
\begin{equation} \label{eq: proof of lemma 5-2}
   \BFcalP_{\Omega(\BFcalE_0)}(\BFcalQ)=\gamma \normaltt{sign}(\BFcalE_0) \textnormal{ and }
   \|\BFcalP_{\Omega(\BFcalE_0)^{\complement}}(\BFcalQ)\|_{\infty} \leq \gamma.
\end{equation}
 Therefore, \eqref{eq: proof of lemma 5-1} and \eqref{eq: proof of lemma 5-2} are necessary and sufficient conditions for $(\BFcalL_0, \BFcalE_0)$ to be a minimizer of \eqref{eq: convex RTPCA}. Hence, $(\BFcalL_0, \BFcalE_0)$ is a minimizer of \eqref{eq: convex RTPCA} with the conditions given in Lemma \ref{lemma: sufficient condition for exact recovery}. Next, we show that $(\BFcalL_0, \BFcalE_0)$ is also a unique minimizer. To avoid cluttered notation, in the rest of this subsection, we denote $T=T(\BFcalL_0)$, $T^{\perp}=T(\BFcalL_0)^{\perp}$, $\Omega=\Omega(\BFcalE_0)$,  and $\Omega^{\complement}=\Omega(\BFcalE_0)^{\complement}$. 
 
We consider a feasible perturbation $(\BFcalL_0 + \BFcalH, \BFcalE_0 - \BFcalH)$ and show that the objective increases whenever $\BFcalH\neq \bm{0}$, hence proving that $(\BFcalL_0 , \BFcalE_0 )$ is the unique minimizer. To do this, let $(\BFcalQ_{\BFcalL},\BFcalQ_{\BFcalE})$ be an arbitrary subgradient  of the objective function in \eqref{eq: convex RTPCA} at $(\BFcalL_0 , \BFcalE_0 )$. By the definition of subgradients,
\begin{equation}  \label{eq: proof of lemma 5-3}
    \|\BFcalL_0+\BFcalH\|_*+\gamma \|\BFcalE_0 - \BFcalH\|_1\geq \|\BFcalL_0\|_*+\gamma \|\BFcalE_0\|_1 + \langle \BFcalQ_{\BFcalL},\BFcalH  \rangle - \langle \BFcalQ_{\BFcalE},\BFcalH  \rangle
\end{equation}
 Since $(\BFcalQ_{\BFcalL},\BFcalQ_{\BFcalE})$ is a subgradient  of the objective function in \eqref{eq: convex RTPCA} at $(\BFcalL_0 , \BFcalE_0 )$, the conditions in \eqref{eq: proof of lemma 5-1} and \eqref{eq: proof of lemma 5-2} hold, namely,
\begin{itemize}
       \item  $\BFcalQ_{\BFcalL}=\BFcalU * \BFcalV^{\top} + \BFcalP_{T^{\perp}}(\BFcalQ_{\BFcalL})$, with $\|\BFcalP_{T^{\perp}}(\BFcalQ_{\BFcalL})\| \leq 1$.
     \item   $\BFcalQ_{\BFcalE}=\gamma \normaltt{sign}(\BFcalE_0) + \BFcalP_{\Omega^{\complement}}(\BFcalQ_{\BFcalE})$, with $\|\BFcalP_{\Omega^{\complement}}(\BFcalQ_{\BFcalE})\|_{\infty} \leq \gamma.$
\end{itemize}
Given the existence of the dual $\BFcalQ$ in Lemma \ref{lemma: sufficient condition for exact recovery}, we have that
\begin{align} \label{eq: proof of lemma 5-4}
    \begin{aligned}
     \langle \BFcalQ_{\BFcalL},\BFcalH  \rangle & =  \langle \BFcalU * \BFcalV^{\top} + \BFcalP_{T^{\perp}}(\BFcalQ_{\BFcalL}),\BFcalH  \rangle \\
     & =  \langle \BFcalQ -  \BFcalP_{T^{\perp}}(\BFcalQ) + \BFcalP_{T^{\perp}}(\BFcalQ_{\BFcalL}),\BFcalH  \rangle \\
     & =  \langle  \BFcalP_{T^{\perp}}(\BFcalQ_{\BFcalL}) -  \BFcalP_{T^{\perp}}(\BFcalQ) ,\BFcalH  \rangle + \langle \BFcalQ ,\BFcalH  \rangle 
     \end{aligned}
\end{align}
Similarly, we have that 
\begin{align}  \label{eq: proof of lemma 5-5}
    \begin{aligned}
     \langle \BFcalQ_{\BFcalE},\BFcalH  \rangle & =  \langle \gamma \normaltt{sign}(\BFcalE_0) + \BFcalP_{\Omega^{\complement}}(\BFcalQ_{\BFcalE}),\BFcalH  \rangle \\
     & =  \langle \BFcalQ -  \BFcalP_{\Omega^{\complement}}(\BFcalQ)+ \BFcalP_{\Omega^{\complement}}(\BFcalQ_{\BFcalE}),\BFcalH  \rangle \\
     & =  \langle \BFcalP_{\Omega^{\complement}}(\BFcalQ_{\BFcalE})  -  \BFcalP_{\Omega^{\complement}}(\BFcalQ),\BFcalH  \rangle + \langle \BFcalQ ,\BFcalH  \rangle 
     \end{aligned}
\end{align}
Putting  \eqref{eq: proof of lemma 5-4} and \eqref{eq: proof of lemma 5-5} together with \eqref{eq: proof of lemma 5-3}, we have the following inequality
\begin{align}\label{eq: proof of lemma 5-6}
    \begin{aligned}
       & \|\BFcalL_0+\BFcalH\|_*+\gamma \|\BFcalE_0 - \BFcalH\|_1 \\
       \geq &  \|\BFcalL_0\|_*+\gamma \|\BFcalE_0\|_1 + \langle \BFcalP_{T^{\perp}}(\BFcalQ_{\BFcalL}) -  \BFcalP_{T^{\perp}}(\BFcalQ) ,\BFcalH  \rangle - \langle \BFcalP_{\Omega^{\complement}}(\BFcalQ_{\BFcalE})  -  \BFcalP_{\Omega^{\complement}}(\BFcalQ),\BFcalH  \rangle \\
       = &  \|\BFcalL_0\|_*+\gamma \|\BFcalE_0\|_1 + \langle \BFcalP_{T^{\perp}}(\BFcalQ_{\BFcalL}) -  \BFcalP_{T^{\perp}}(\BFcalQ) ,\BFcalP_{T^{\perp}}(\BFcalH)  \rangle - \langle \BFcalP_{\Omega^{\complement}}(\BFcalQ_{\BFcalE})  -  \BFcalP_{\Omega^{\complement}}(\BFcalQ),\BFcalP_{\Omega^{\complement}}(\BFcalH)  \rangle.
     \end{aligned}
\end{align}  
Let the t-SVD of $\BFcalP_{T^{\perp}}(\BFcalH)$ as $\tilde{\BFcalU} * \tilde{\BFcalS} * \tilde{\BFcalV}^{\top}$, we set $\BFcalP_{T^{\perp}}(\BFcalQ_{\BFcalL})=\tilde{\BFcalU} * \tilde{\BFcalV}^{\top}$ such that $\|\BFcalP_{T^{\perp}}(\BFcalQ_{\BFcalL})\|=1$ and $\langle \BFcalP_{T^{\perp}}(\BFcalQ_{\BFcalL}) ,\BFcalP_{T^{\perp}}(\BFcalH)  \rangle = \|\BFcalP_{T^{\perp}}(\BFcalH) \|_*$, which satisfy the conditions in \eqref{eq: proof of lemma 5-1}. Furthermore, we set $ \BFcalP_{\Omega^{\complement}}(\BFcalQ_{\BFcalE})=-\gamma \normaltt{sign}(\BFcalP_{\Omega^{\complement}}(\BFcalH))$ such that $\|\BFcalP_{\Omega^{\complement}}(\BFcalQ_{\BFcalE})\|_{\infty}=\gamma$ and $\langle \BFcalP_{\Omega^{\complement}}(\BFcalQ_{\BFcalE}) ,\BFcalP_{\Omega^{\complement}}(\BFcalH)  \rangle =- \gamma\|\BFcalP_{\Omega^{\complement}}(\BFcalH)\|_1$, which satisfy the conditions in \eqref{eq: proof of lemma 5-2}. Based on the carefully selected $\BFcalP_{T^{\perp}}(\BFcalQ_{\BFcalL})$ and  $\BFcalP_{\Omega^{\complement}}(\BFcalQ_{\BFcalE})$, the inequality \eqref{eq: proof of lemma 5-6} can be simplified as
\begin{align}\label{eq: proof of lemma 5-7}
    \begin{aligned}
       & \|\BFcalL_0+\BFcalH\|_*+\gamma \|\BFcalE_0 - \BFcalH\|_1 \\
       \geq &  \|\BFcalL_0\|_*+\gamma \|\BFcalE_0\|_1 + \langle \BFcalP_{T^{\perp}}(\BFcalQ_{\BFcalL}) -  \BFcalP_{T^{\perp}}(\BFcalQ) ,\BFcalP_{T^{\perp}}(\BFcalH)  \rangle - \langle \BFcalP_{\Omega^{\complement}}(\BFcalQ_{\BFcalE})  -  \BFcalP_{\Omega^{\complement}}(\BFcalQ),\BFcalP_{\Omega^{\complement}}(\BFcalH)  \rangle \\
       =&  \|\BFcalL_0\|_*+\gamma \|\BFcalE_0\|_1 + \|\BFcalP_{T^{\perp}}(\BFcalH) \|_* - \langle  \BFcalP_{T^{\perp}}(\BFcalQ) ,\BFcalP_{T^{\perp}}(\BFcalH)  \rangle  +  \gamma\|\BFcalP_{\Omega^{\complement}}(\BFcalH)\|_1 +   \langle  \BFcalP_{\Omega^{\complement}}(\BFcalQ),\BFcalP_{\Omega^{\complement}}(\BFcalH)  \rangle \\
       \geq &  \|\BFcalL_0\|_*+\gamma \|\BFcalE_0\|_1 +(1 - \|\BFcalP_{T^{\perp}}(\BFcalQ) \|)\|\BFcalP_{T^{\perp}}(\BFcalH) \|_* + (\gamma -\|\BFcalP_{\Omega^{\complement}}(\BFcalQ)\|_{\infty} )\|\BFcalP_{\Omega^{\complement}}(\BFcalH)\|_1 \\
       > &  \|\BFcalL_0\|_*+\gamma \|\BFcalE_0\|_1,
     \end{aligned}
\end{align} 
where the first inequality is based on carefully selected $\BFcalP_{T^{\perp}}(\BFcalQ_{\BFcalL})$ and  $\BFcalP_{\Omega^{\complement}}(\BFcalQ_{\BFcalE})$, the second inequality  is due to the fact that $|\langle  \BFcalP_{T^{\perp}}(\BFcalQ) ,\BFcalP_{T^{\perp}}(\BFcalH)  \rangle|\leq \|\BFcalP_{T^{\perp}}(\BFcalQ) \|\cdot\|\BFcalP_{T^{\perp}}(\BFcalH)\|_*$ and  $| \langle  \BFcalP_{\Omega^{\complement}}(\BFcalQ),\BFcalP_{\Omega^{\complement}}(\BFcalH)  \rangle |\leq \|\BFcalP_{\Omega^{\complement}}(\BFcalQ)\|_{\infty} \cdot\|\BFcalP_{\Omega^{\complement}}(\BFcalH)\|_1$, and the last inequality is due to the given conditions in Lemma \ref{lemma: sufficient condition for exact recovery}, namely, $\|\BFcalP_{T(\BFcalL_0)^{\perp}}(\BFcalQ)\|< 1$ and  $\|\BFcalP_{\Omega(\BFcalE_0)^{\complement}}(\BFcalQ)\|_{\infty}<\gamma$, and $\|\BFcalP_{T^{\perp}}(\BFcalH) \|_* +\|\BFcalP_{\Omega^{\complement}}(\BFcalH)\|_1 >0$, which is derived from $\BFcalH\neq \bm{0}$ and the given condition $T(\BFcalL_0) \cap \Omega(\BFcalE_0)=\{\bm{0}\}$.
 
The above inequality \eqref{eq: proof of lemma 5-7} leads to the statement that $\|\BFcalL_0+\BFcalH\|_*+\gamma \|\BFcalE_0 - \BFcalH\|_1>\|\BFcalL_0\|_*+\gamma \|\BFcalE_0\|_1$ unless $\BFcalH=\bm{0}$. Thus, the uniqueness of $(\BFcalL_0,\BFcalE_0)$ is proved here.
 \end{proof}
 
\begin{proof}[Proof of Theorem \ref{thm: exact recovery 1}]
The proof of Theorem \ref{thm: exact recovery 1} can be viewed as a dual certification. That is, given the condition in Theorem \ref{thm: exact recovery 1}, there exist a dual $\BFcalQ$ satisfying the sufficient conditions provided in Lemma \ref{lemma: sufficient condition for exact recovery}. We aims to construct a dual $\BFcalQ$ by considering candidates in the direct sum $T\bigoplus\Omega$ of the tangent spaces. Since $\xi(\BFcalL_0)\mu(\BFcalE_0) < \frac{1}{6}$, we can conclude from Lemma \ref{lemma: sufficient for trivial intersection} that there exist a unique $\hat{\BFcalQ}$ such that $\BFcalP_{T}(\hat{\BFcalQ})=\BFcalU*\BFcalV^{\top}$ and $\BFcalP_{\Omega}(\hat{\BFcalQ})=\gamma\normaltt{sign}(\BFcalE_0)$. The rest of this proof shows that if $\xi(\BFcalL_0)\mu(\BFcalE_0) < \frac{1}{6}$, then the projections of such a $\hat{\BFcalQ}$ onto $T^{\perp}$ and onto $\Omega^{\complement}$ will be small, namely,  $\|\BFcalP_{T(\BFcalL_0)^{\perp}}(\BFcalQ)\|< 1$ and  $\|\BFcalP_{\Omega(\BFcalE_0)^{\complement}}(\BFcalQ)\|_{\infty}<\gamma$.

Note that $\hat{\BFcalQ}$ can be uniquely decomposed into two parts, namely, an element of $T$ and an element of $\Omega$, which can be expressed as $\hat{\BFcalQ}=\BFcalQ_{T}+\BFcalQ_{\Omega}$ where $\BFcalQ_{T}\in T$ and $\BFcalQ_{\Omega}\in \Omega$. Let  $\BFcalQ_{T}=\BFcalU * \BFcalV^{\top} +\BFcalH_{T}$ and    $\BFcalQ_{\Omega}=\gamma\normaltt{sign}(\BFcalE_0)+\BFcalH_{\Omega}$. Accordingly, 
\begin{equation*}
    \BFcalP_{T}(\hat{\BFcalQ})  = \BFcalU * \BFcalV^{\top} +\BFcalH_{T} + \BFcalP_{T}(\gamma\normaltt{sign}(\BFcalE_0)+\BFcalH_{\Omega}).
\end{equation*}
Since $\BFcalP_{T}(\hat{\BFcalQ})  = \BFcalU * \BFcalV^{\top}$, the below can be obtained
\begin{equation*} \label{eq: proof of thm 3-1}
    \BFcalH_{T} = - \BFcalP_{T}(\gamma\normaltt{sign}(\BFcalE_0)+\BFcalH_{\Omega}).
\end{equation*}
Similarly,
\begin{equation*} \label{eq: proof of thm 3-2}
    \BFcalH_{\Omega} = - \BFcalP_{\Omega}(\BFcalU * \BFcalV^{\top} +\BFcalH_{T}).
\end{equation*}
Next,
\begin{align} \label{eq: proof of thm 3-3}
    \begin{aligned}
   \|\BFcalP_{T^{\perp}}(\hat{\BFcalQ})\| & = \|\BFcalP_{T^{\perp}}(\gamma\normaltt{sign}(\BFcalE_0)+\BFcalH_{\Omega})\| \\
   & \leq  \|\gamma\normaltt{sign}(\BFcalE_0)+\BFcalH_{\Omega}\| \\
   & \leq \mu(\BFcalE_0)\|\gamma\normaltt{sign}(\BFcalE_0)+\BFcalH_{\Omega}\|_{\infty} \\
   & \leq \mu(\BFcalE_0)(\gamma + \|\BFcalH_{\Omega}\|_{\infty}),
     \end{aligned}
\end{align}
where the first inequality is obvious, the second inequality is based on the definition of $\mu(\BFcalE_0)$ since $\gamma\normaltt{sign}(\BFcalE_0)+\BFcalH_{\Omega} \in \Omega$. The bound $\|\BFcalP_{\Omega^{\complement}}(\hat{\BFcalQ})\|_{\infty}$ can be obtained as follows
\begin{align} \label{eq: proof of thm 3-4}
    \begin{aligned}
   \|\BFcalP_{\Omega^{\complement}}(\hat{\BFcalQ})\|_{\infty} & =  \|\BFcalP_{\Omega^{\complement}}(\BFcalU * \BFcalV^{\top} +\BFcalH_{T})\|_{\infty} \\
   & \leq \|\BFcalU * \BFcalV^{\top} +\BFcalH_{T}\|_{\infty}  \\
   & \leq \xi(\BFcalL_0) \|\BFcalU * \BFcalV^{\top} +\BFcalH_{T}\| \\
   & \leq  \xi(\BFcalL_0)(1 +\|\BFcalH_{T}\| ),
     \end{aligned}
\end{align}
where the first inequality is obvious, the second inequality is based on the definition of $\xi(\BFcalL_0)$ since $\BFcalU * \BFcalV^{\top} +\BFcalH_{T} \in T$, and the last inequality is due to the triangle inequality for tensor spectral norm. Furthermore, the bounds for $\|\BFcalH_{T}\|$ and  $\|\BFcalH_{\Omega}\|_{\infty}$ are derived, respectively.
\begin{align} \label{eq: proof of thm 3-5}
    \begin{aligned}
  \|\BFcalH_{T}\| & =  \|\BFcalP_{T}(\gamma\normaltt{sign}(\BFcalE_0)+\BFcalH_{\Omega})\| \\
   & \leq  2\|\gamma\normaltt{sign}(\BFcalE_0)+\BFcalH_{\Omega}\| \\
   & \leq 2\mu(\BFcalE_0)\|\gamma\normaltt{sign}(\BFcalE_0)+\BFcalH_{\Omega}\|_{\infty} \\
   & \leq  2\mu(\BFcalE_0)(\gamma +\|\BFcalH_{\Omega}\|_{\infty} ),
     \end{aligned}
\end{align}
where the first inequality is due to the fact that $\|\BFcalP_T(\BFcalA)\|\leq 2\|\BFcalA\|$, the second inequality is based on the definition $\mu(\BFcalE_0)$, and the last inequality is obtained by the triangle inequality for $\|\cdot\|_{\infty}$. Similarly, 
\begin{align} \label{eq: proof of thm 3-6}
    \begin{aligned}
  \|\BFcalH_{\Omega}\|_{\infty} & =  \|\BFcalP_{\Omega}(\BFcalU * \BFcalV^{\top} +\BFcalH_{T})\|_{\infty} \\
   & \leq \|\BFcalU * \BFcalV^{\top} +\BFcalH_{T}\|_{\infty} \\
   & \leq \xi(\BFcalL_0)\|\BFcalU * \BFcalV^{\top} +\BFcalH_{T}\| \\
   & \leq   \xi(\BFcalL_0)(1 +\|\BFcalH_{T}\|),
     \end{aligned}
\end{align}
where the first inequality is obvious, the second inequality is based on the definition $\xi(\BFcalL_0)$, and the last inequality is obtained by the triangle inequality for tensor spectral norm $\|\cdot\|$.

Plugging \eqref{eq: proof of thm 3-6} into \eqref{eq: proof of thm 3-5}, we have that
\begin{align} \label{eq: proof of thm 3-7}
    \begin{aligned}
  & \|\BFcalH_{T}\| \leq   2\mu(\BFcalE_0)(\gamma + \xi(\BFcalL_0)(1 +\|\BFcalH_{T}\|) )\\
  \Rightarrow &  \|\BFcalH_{T}\|  \leq \frac{2\gamma\mu(\BFcalE_0)+ 2\xi(\BFcalL_0)\mu(\BFcalE_0)}{1-2\xi(\BFcalL_0)\mu(\BFcalE_0)}. \end{aligned}
\end{align}
Inversely, plugging \eqref{eq: proof of thm 3-5} into \eqref{eq: proof of thm 3-6},
\begin{align} \label{eq: proof of thm 3-8}
    \begin{aligned}
  &  \|\BFcalH_{\Omega}\|_{\infty} \leq   \xi(\BFcalL_0)(1 + 2\mu(\BFcalE_0)(\gamma +\|\BFcalH_{\Omega}\|_{\infty} ) )\\
  \Rightarrow & \|\BFcalH_{\Omega}\|_{\infty}   \leq \frac{\xi(\BFcalL_0)+ 2\gamma\xi(\BFcalL_0)\mu(\BFcalE_0)}{1-2\xi(\BFcalL_0)\mu(\BFcalE_0)}. \end{aligned}
\end{align}
Now we can bound $ \|\BFcalP_{T^{\perp}}(\hat{\BFcalQ})\|$ by combining \eqref{eq: proof of thm 3-3} and \eqref{eq: proof of thm 3-8},
\begin{align*}
    \begin{aligned}
   \|\BFcalP_{T^{\perp}}(\hat{\BFcalQ})\| &  \leq \mu(\BFcalE_0)(\gamma + \frac{\xi(\BFcalL_0)+ 2\gamma\xi(\BFcalL_0)\mu(\BFcalE_0)}{1-2\xi(\BFcalL_0)\mu(\BFcalE_0)})\\
   & =  \mu(\BFcalE_0)  \frac{\gamma + \xi(\BFcalL_0)}{1-2\xi(\BFcalL_0)\mu(\BFcalE_0)} \\
   & <  \mu(\BFcalE_0)\frac{\frac{1 - 3\xi(\BFcalL_0)\mu(\BFcalE_0)}{\mu(\BFcalE_0)} + \xi(\BFcalL_0)}{1-2\xi(\BFcalL_0)\mu(\BFcalE_0)} \\
   & = 1,
     \end{aligned}
\end{align*}
where the first inequality is given by plugging \eqref{eq: proof of thm 3-8} into \eqref{eq: proof of thm 3-3}, the second inequality is due to the assumption in Theorem \ref{thm: exact recovery 1}, namely, $\gamma < \frac{1 - 3\xi(\BFcalL_0)\mu(\BFcalE_0)}{\mu(\BFcalE_0)}$.

In the end of this proof, the bound for $ \|\BFcalP_{\Omega^{\complement}}(\hat{\BFcalQ})\|_{\infty} $ can be obtained by combining by plugging \eqref{eq: proof of thm 3-7} into \eqref{eq: proof of thm 3-4},
\begin{align*} 
    \begin{aligned}
   \|\BFcalP_{\Omega^{\complement}}(\hat{\BFcalQ})\|_{\infty} & \leq \xi(\BFcalL_0)(1 + \frac{2\gamma\mu(\BFcalE_0)+ 2\xi(\BFcalL_0)\mu(\BFcalE_0)}{1-2\xi(\BFcalL_0)\mu(\BFcalE_0)} )\\
   & = \xi(\BFcalL_0)\frac{1+2\gamma\mu(\BFcalE_0)}{1-2\xi(\BFcalL_0)\mu(\BFcalE_0)} \\
   & = \left[\xi(\BFcalL_0)\frac{1+2\gamma\mu(\BFcalE_0)}{1-2\xi(\BFcalL_0)\mu(\BFcalE_0)}  - \gamma\right] +\gamma \\
   & = \left [\frac{\xi(\BFcalL_0)-\gamma(1-4\xi(\BFcalL_0)\mu(\BFcalE_0))}{1-2\xi(\BFcalL_0)\mu(\BFcalE_0)}\right] +\gamma \\
   & < \left [\frac{\xi(\BFcalL_0)-\xi(\BFcalL_0)}{1-2\xi(\BFcalL_0)\mu(\BFcalE_0)}\right] +\gamma  \\
   & = \gamma,
     \end{aligned}
\end{align*}
where the last inequality is due to the assumption in Theorem \ref{thm: exact recovery 1}, that is, $\gamma > \frac{\xi(\BFcalL_0)}{1-4\xi(\BFcalL_0)\mu(\BFcalE_0)}$. Hence, the dual certification is done.

In addition, we can verify the lower and upper bounds for $\gamma$ is feasible through the following claim 
\begin{equation} \label{eq: proof of thm 3-9}
\xi(\BFcalL_0)\mu(\BFcalE_0) < \frac{1}{6}  \Rightarrow \frac{\xi(\BFcalL_0)}{1-4\xi(\BFcalL_0)\mu(\BFcalE_0)}<\frac{1-3\xi(\BFcalL_0)\mu(\BFcalE_0)}{\mu(\BFcalE_0)},
\end{equation}
which can be obtained by computing the roots for a quadratic function. For any $p\in [0,1]$, we can verify that $\gamma=\frac{(3\xi(\BFcalL_0))^p}{(2\mu(\BFcalE_0))^{1-p}}$ is always inside the above range. 
\end{proof}
\subsection{Proofs of Lemmas \ref{lemma: low-rank tensor property} and \ref{lemma: sparsity pattern}} In this subsection, the proofs for the bounds for $\xi(\BFcalL)$ and $\mu(\BFcalE)$ by characterizing $\normaltt{inc}(\BFcalL)$ and $\normaltt{deg}_{min}(\BFcalE), \normaltt{deg}_{max}(\BFcalE)$ are provided. 
\begin{proof}[Proof of Lemma \ref{lemma: low-rank tensor property}] 
Given t-SVD $\BFcalL =\BFcalU* \BFcalS * \BFcalV$, 
\begin{align} \label{eq: proof of lemma 3-1}
      \begin{aligned}
      \xi(\BFcalL) & = \underset{\BFcalN \in T(\BFcalL),\|\BFcalN\|\leq 1}{\text{max}} \|\BFcalN\|_{\infty}\\
          & = \underset{\BFcalN \in T(\BFcalL),\|\BFcalN\|\leq 1}{\text{max}} \|\BFcalP_{T(\BFcalL)}(\BFcalN)\|_{\infty} \\
          & \leq \underset{\|\BFcalN\|\leq 1}{\text{max}} \|\BFcalP_{T(\BFcalL)}(\BFcalN)\|_{\infty} \\
           & \leq \underset{\BFcalN \textnormal{orthogonal}}{\text{max}} \|\BFcalP_{T(\BFcalL)}(\BFcalN)\|_{\infty} \\
           & \leq \underset{\BFcalN \textnormal{orthogonal}}{\text{max}} \|\BFcalP_{\BFcalU}*\BFcalN\|_{\infty} + \underset{\BFcalN \textnormal{orthogonal}}{\text{max}} \|(\BFcalI_{N_1}-\BFcalP_{\BFcalU})*\BFcalN * \BFcalP_{\BFcalV}\|_{\infty},
      \end{aligned}
\end{align} 
  where the second inequality is due to the fact that the  maximum of a convex function over a convex set is achieved at one of the extreme points of the constraint set. The orthogonal tensors are the extreme points of the set of contractions. The last inequality is the triangle inequality. We further show the upper bounds for the two terms in the last line  of \eqref{eq: proof of lemma 3-1}, namely, $\underset{\BFcalN \textnormal{orthogonal}}{\text{max}} \|\BFcalP_{\BFcalU}*\BFcalN\|_{\infty}$ and $\underset{\BFcalN \textnormal{orthogonal}}{\text{max}} \|(\BFcalI_{N_1}-\BFcalP_{\BFcalU})*\BFcalN * \BFcalP_{\BFcalV}\|_{\infty}$. 
\begin{align} \label{eq: proof of lemma 3-2}
      \begin{aligned}
     \underset{\BFcalN \textnormal{orthogonal}}{\text{max}} \|\BFcalP_{\BFcalU}*\BFcalN\|_{\infty} & =   \underset{\BFcalN \textnormal{orthogonal}}{\text{max}} \; \underset{n_1,n_2,n_3}{\text{max}} |\mathring{\mathpzc{e}}_{n_1}^{\top}*\BFcalP_{\BFcalU}*\BFcalN * \mathring{\mathpzc{e}}_{n_2}(1,1,n_3)|\\
      & \leq   \underset{\BFcalN \textnormal{orthogonal}}{\text{max}} \; \underset{n_1,n_2}{\text{max}} \|\mathring{\mathpzc{e}}_{n_1}^{\top}*\BFcalP_{\BFcalU}\|_F\|\BFcalN * \mathring{\mathpzc{e}}_{n_2}\|_F \\
      & = \underset{n_1}{\text{max}} \|\BFcalP_{\BFcalU} * \mathring{\mathpzc{e}}_{n_1}\|_F\times \underset{\BFcalN \textnormal{orthogonal}}{\text{max}} \; \underset{n_2}{\text{max}} \|\BFcalN * \mathring{\mathpzc{e}}_{n_2}\|_F\\
      & = \beta(\BFcalU),
      \end{aligned}
\end{align} 
where the first inequality is based on the property of t-SVD and the famous  Cauchy–Schwarz inequality (Note that $\mathring{\mathpzc{e}}_{n_1}^{\top}*\BFcalP_{\BFcalU}*\BFcalN * \mathring{\mathpzc{e}}_{n_2}$ is a tensor with size $1\times 1\times N_3$).
\begin{align} \label{eq: proof of lemma 3-3}
      \begin{aligned}
  & \underset{\BFcalN \textnormal{orthogonal}}{\text{max}} \|(\BFcalI_{N_1}-\BFcalP_{\BFcalU})*\BFcalN * \BFcalP_{\BFcalV}\|_{\infty} \\
 = &    \underset{\BFcalN \textnormal{orthogonal}}{\text{max}} \; \underset{n_1,n_2,n_3}{\text{max}} |\mathring{\mathpzc{e}}_{n_1}^{\top}*(\BFcalI_{N_1}-\BFcalP_{\BFcalU})*\BFcalN * \BFcalP_{\BFcalV}* \mathring{\mathpzc{e}}_{n_2}(1,1,n_3)|\\
    \leq  &    \underset{\BFcalN \textnormal{orthogonal}}{\text{max}} \; \underset{n_1,n_2}{\text{max}} \|\mathring{\mathpzc{e}}_{n_1}^{\top}*(\BFcalI_{N_1}-\BFcalP_{\BFcalU})\|_F\|\BFcalN * \BFcalP_{\BFcalV}* \mathring{\mathpzc{e}}_{n_2}\|_F \\
     = &  \underset{n_1}{\text{max}} \|\mathring{\mathpzc{e}}_{n_1}^{\top}*(\BFcalI_{N_1}-\BFcalP_{\BFcalU})\|_F \times \underset{\BFcalN \textnormal{orthogonal}}{\text{max}} \; \underset{n_2}{\text{max}} \|\BFcalN * \BFcalP_{\BFcalV}* \mathring{\mathpzc{e}}_{n_2}\|_F\\
     \leq & 1 \times  \underset{n_2}{\text{max}} \|\BFcalP_{\BFcalV}* \mathring{\mathpzc{e}}_{n_2}\|_F\\
  =    &  \beta(\BFcalV),
      \end{aligned}
\end{align} 
where the first inequality is based on the property of t-SVD and the famous  Cauchy–Schwarz inequality (Note that $\mathring{\mathpzc{e}}_{n_1}^{\top}*(\BFcalI_{N_1}-\BFcalP_{\BFcalU})*\BFcalN * \BFcalP_{\BFcalV}* \mathring{\mathpzc{e}}_{n_2}$ is a tensor with size $1\times 1\times N_3$). Then, plugging \eqref{eq: proof of lemma 3-2} and \eqref{eq: proof of lemma 3-3} into \eqref{eq: proof of lemma 3-1},
\begin{equation}
     \xi(\BFcalL) \leq \beta(\BFcalU) + \beta(\BFcalV) \leq 2 \normaltt{inc}(\BFcalL).
\end{equation}
Hence the upper bound on $ \xi(\BFcalL) $ is derived. Next, the lower bound on $ \xi(\BFcalL) $ is further developed. By verifying that $\|\BFcalP_{\BFcalU}*\BFcalN\|\leq 1$, we have that 
\begin{align*}
    \begin{aligned}
    \xi(\BFcalL) &\geq \underset{\BFcalN \textnormal{orthogonal}}{\text{max}} \|\BFcalP_{\BFcalU}*\BFcalN\|_{\infty} \\
  &\geq  \frac{1}{\sqrt{N_3}}  \underset{\BFcalN \textnormal{orthogonal}}{\text{max}} \; \underset{n_1,n_2}{\text{max}} \|\mathring{\mathpzc{e}}_{n_1}^{\top}*\BFcalP_{\BFcalU}*\BFcalN * \mathring{\mathpzc{e}}_{n_2}\|_F  \\
  & =  \frac{1}{\sqrt{N_3}}  \underset{\BFcalN \textnormal{orthogonal}}{\text{max}} \; \underset{n_1,n_2}{\text{max}} \|\mathring{\mathpzc{e}}_{n_1}^{\top}*\BFcalP_{\BFcalU}\|_F\|\BFcalN * \mathring{\mathpzc{e}}_{n_2}\|_F\\
  & = \frac{1}{\sqrt{N_3}}\underset{n_1}{\text{max}} \|\BFcalP_{\BFcalU} * \mathring{\mathpzc{e}}_{n_1}\|_F\times \underset{\BFcalN \textnormal{orthogonal}}{\text{max}} \; \underset{n_2}{\text{max}} \|\BFcalN * \mathring{\mathpzc{e}}_{n_2}\|_F\\
      & = \frac{1}{\sqrt{N_3}}\beta(\BFcalU),
    \end{aligned}
\end{align*}
where the first inequality is based on $ \xi(\BFcalL)$, the second inequality is based on the relationship between tensor infinity norm and tensor Frobenius norm, and the first equality can be achieved by setting orthogonal tensor $\BFcalN$ with one of its slice equal to $\frac{1}{\beta(\BFcalU)}\BFcalP_{\BFcalU}\mathring{\mathpzc{e}}_{n_1^*}$ with $n_1^*$ is the index to achieve $\beta(\BFcalU)=\|\BFcalP_{\BFcalU} * \mathring{\mathpzc{e}}_{n_1^*}\|_2$. Similarly, we have the same argument with respect to $\BFcalV$. Therefore, the lower bound is proved.
\end{proof}

\begin{proof}[Proof of Lemma \ref{lemma: sparsity pattern}] 
Based on the Perron–Frobenius theorem \cite{horn2012matrix}, one can conclude that for a matrix $\|\BFA\|\geq\|\BFB\|$ if $\BFA\geq |\BFB|$ in an element-wise fashion. Thus, we need only consider the tensor that has $1$ in every location in the support set $\Omega(\BFcalE)$ and $0$ everywhere else. Based on the definition of the spectral norm, we can rewrite $\mu(\BFcalE)$ as follows:
\begin{equation} \label{eq: proof of lemma 4-1}
    \mu(\BFcalE) = \underset{\|\BFx\|_2=1,\|\BFy\|_2=1}{\text{max}} \sum_{(i,j)\in \Omega(\normaltt{bcirc}(\BFcalE))} x_iy_j,
\end{equation}
where $\BFx\in \mathbb{R}^{N_1\times N_3}$ and $\BFy\in \mathbb{R}^{N_2\times N_3}$. Note that the above equality is due to the fact that $\|\BFcalE\|=\|\normaltt{bcirc}(\BFcalE)\|$ so that we can transform the tensor spectral norm into a matrix spectral norm. Let $\normaltt{bcirc}(\BFcalA)^{\Omega(\BFcalE)}$ be a matrix defined as follows:
$$\normaltt{bcirc}(\BFcalA)^{\Omega(\BFcalE)}(i,j)=
\begin{cases}
1& (i,j)\in \Omega(\normaltt{bcirc}(\BFcalA)),\\
0& \text{otherwise}.
\end{cases}$$
Based on \eqref{eq: proof of lemma 4-1}, we can obtain the following
$$\mu(\BFcalE) =\|\normaltt{bcirc}(\BFcalA)^{\Omega(\BFcalE)}\|.$$

Next, the upper bound will be derived. For any tensor $\BFcalA$, we have the following \cite{schur1911bemerkungen}
\begin{equation} \label{eq: proof of lemma 4-2}
 \|\BFcalA\|^2=\|\normaltt{bcirc}(\BFcalA)\|^2\leq \underset{i,j}{\text{max}} \; r_i c_j,
\end{equation}
where $r_i=\sum_{k}|\normaltt{bcirc}(\BFcalA)(i,k)|$ denotes the absolute row-sum of row $i$ and $c_j=\sum_{k}|\normaltt{bcirc}(\BFcalA)(k,j)|$ denotes the absolute column-sum of column $j$.  Note that based on the definition in \eqref{eq: bcirc}, $r_i$ and $c_j$ are nonzero entries in $i$-th horizontal and $j$-th lateral slices, respectively. According to the bound \eqref{eq: proof of lemma 4-2}, the upper bound can be obtained
\[\|\BFcalA\|=\|\normaltt{bcirc}(\BFcalA)^{\Omega(\BFcalE)}\| \leq \normaltt{deg}_{max}(\BFcalE).\]
 Given that per horizontal/lateral slice of $\BFcalE$ has at least $\normaltt{deg}_{min}(\BFcalE)$ nonzero entries. Now, we derive the lower bound for $\mu(\BFcalE)$,
 \[\mu(\BFcalE)\geq \sum_{(i,j)\in \Omega(\normaltt{bcirc}(\BFcalE))} \frac{1}{\sqrt{N_1N_3}} \frac{1}{\sqrt{N_2N_3}}=\frac{N_3|\normaltt{support}(\BFcalE)|}{\sqrt{N_1N_2}N_3}\geq  \normaltt{deg}_{min}(\BFcalE),\]
where we set $\BFx=\frac{1}{\sqrt{N_1N_3}}\bm{1}$ and $\BFy=\frac{1}{\sqrt{N_2N_3}}\bm{1}$ with $\bm{1}$ representing the all-ones vector, which are feasible points for the optimization in \eqref{eq: proof of lemma 4-1}.
 \end{proof}
\section{Conclusion} \label{sec: conclusion}
In this paper, we studied the problem of exact recovery of the low tubal rank tensor, namely, $\BFcalL_0$ and the sparse tensor, namely, $\BFcalE_0$, from the observation $\BFcalX =\BFcalL_0 + \BFcalE_0$ via the convex optimization \eqref{eq: convex RTPCA}. It is a popular problem arsing in many machine learning applications such  as moving object tracking, image recovery, and background modeling.  Using t-SVD, The tensor spectral norm, tensor nuclear norm, and tensor tubal rank are developed such that their properties and relationships are consistent with the matrix cases.  The deterministic  sufficient condition for the exact recovery is provided without assuming the uniform model on the sparse support.  

An interesting problem for further research is to extend the deterministic analysis to the tensor completion problem, which aims to recover the low-rank tensor from its partial observed structure. Beyond the convex models, the extensions to non-convex cases are also important, which deserve further investigation.
\medskip

\bibliographystyle{unsrtnat}
\bibliography{Bibliography}

\begin{thebibliography}{23}
\providecommand{\natexlab}[1]{#1}
\providecommand{\url}[1]{\texttt{#1}}
\expandafter\ifx\csname urlstyle\endcsname\relax
  \providecommand{\doi}[1]{doi: #1}\else
  \providecommand{\doi}{doi: \begingroup \urlstyle{rm}\Url}\fi

\bibitem[Sidiropoulos et~al.(2017)Sidiropoulos, De~Lathauwer, Fu, Huang,
  Papalexakis, and Faloutsos]{sidiropoulos2017tensor}
Nicholas~D Sidiropoulos, Lieven De~Lathauwer, Xiao Fu, Kejun Huang, Evangelos~E
  Papalexakis, and Christos Faloutsos.
\newblock Tensor decomposition for signal processing and machine learning.
\newblock \emph{IEEE Transactions on Signal Processing}, 65\penalty0
  (13):\penalty0 3551--3582, 2017.

\bibitem[Cand{\`e}s et~al.(2011)Cand{\`e}s, Li, Ma, and
  Wright]{candes2011robust}
Emmanuel~J Cand{\`e}s, Xiaodong Li, Yi~Ma, and John Wright.
\newblock Robust principal component analysis?
\newblock \emph{Journal of the ACM (JACM)}, 58\penalty0 (3):\penalty0 1--37,
  2011.

\bibitem[Chandrasekaran et~al.(2011)Chandrasekaran, Sanghavi, Parrilo, and
  Willsky]{chandrasekaran2011rank}
Venkat Chandrasekaran, Sujay Sanghavi, Pablo~A Parrilo, and Alan~S Willsky.
\newblock Rank-sparsity incoherence for matrix decomposition.
\newblock \emph{SIAM Journal on Optimization}, 21\penalty0 (2):\penalty0
  572--596, 2011.

\bibitem[Lu et~al.(2016)Lu, Feng, Chen, Liu, Lin, and Yan]{lu2016tensor}
Canyi Lu, Jiashi Feng, Yudong Chen, Wei Liu, Zhouchen Lin, and Shuicheng Yan.
\newblock Tensor robust principal component analysis: Exact recovery of
  corrupted low-rank tensors via convex optimization.
\newblock In \emph{Proceedings of the IEEE conference on computer vision and
  pattern recognition}, pages 5249--5257, 2016.

\bibitem[Lu et~al.(2019)Lu, Feng, Chen, Liu, Lin, and Yan]{lu2019tensor}
Canyi Lu, Jiashi Feng, Yudong Chen, Wei Liu, Zhouchen Lin, and Shuicheng Yan.
\newblock Tensor robust principal component analysis with a new tensor nuclear
  norm.
\newblock \emph{IEEE transactions on pattern analysis and machine
  intelligence}, 42\penalty0 (4):\penalty0 925--938, 2019.

\bibitem[Sobral(2017)]{sobral2017robust}
Andrews~Cordolino Sobral.
\newblock \emph{Robust low-rank and sparse decomposition for moving object
  detection: from matrices to tensors}.
\newblock PhD thesis, UNIVERSITÉ DE LA ROCHELLE, 2017.

\bibitem[Zhou and Feng(2017)]{zhou2017outlier}
Pan Zhou and Jiashi Feng.
\newblock Outlier-robust tensor pca.
\newblock In \emph{Proceedings of the IEEE Conference on Computer Vision and
  Pattern Recognition}, pages 2263--2271, 2017.

\bibitem[Cao et~al.(2016)Cao, Wang, Sun, Meng, Yang, Cichocki, and
  Xu]{cao2016total}
Wenfei Cao, Yao Wang, Jian Sun, Deyu Meng, Can Yang, Andrzej Cichocki, and
  Zongben Xu.
\newblock Total variation regularized tensor rpca for background subtraction
  from compressive measurements.
\newblock \emph{IEEE Transactions on Image Processing}, 25\penalty0
  (9):\penalty0 4075--4090, 2016.

\bibitem[Anandkumar et~al.(2017)Anandkumar, Deng, Ge, and
  Mobahi]{anandkumar2017homotopy}
Anima Anandkumar, Yuan Deng, Rong Ge, and Hossein Mobahi.
\newblock Homotopy analysis for tensor pca.
\newblock In \emph{Conference on Learning Theory}, pages 79--104, 2017.

\bibitem[Kolda and Bader(2009)]{kolda2009tensor}
Tamara~G Kolda and Brett~W Bader.
\newblock Tensor decompositions and applications.
\newblock \emph{SIAM review}, 51\penalty0 (3):\penalty0 455--500, 2009.

\bibitem[Hillar and Lim(2013)]{hillar2013most}
Christopher~J Hillar and Lek-Heng Lim.
\newblock Most tensor problems are np-hard.
\newblock \emph{Journal of the ACM (JACM)}, 60\penalty0 (6):\penalty0 1--39,
  2013.

\bibitem[Kilmer and Martin(2011)]{kilmer2011factorization}
Misha~E Kilmer and Carla~D Martin.
\newblock Factorization strategies for third-order tensors.
\newblock \emph{Linear Algebra and its Applications}, 435\penalty0
  (3):\penalty0 641--658, 2011.

\bibitem[Kilmer et~al.(2013)Kilmer, Braman, Hao, and Hoover]{kilmer2013third}
Misha~E Kilmer, Karen Braman, Ning Hao, and Randy~C Hoover.
\newblock Third-order tensors as operators on matrices: A theoretical and
  computational framework with applications in imaging.
\newblock \emph{SIAM Journal on Matrix Analysis and Applications}, 34\penalty0
  (1):\penalty0 148--172, 2013.

\bibitem[Zhang et~al.(2014)Zhang, Ely, Aeron, Hao, and Kilmer]{zhang2014novel}
Zemin Zhang, Gregory Ely, Shuchin Aeron, Ning Hao, and Misha Kilmer.
\newblock Novel methods for multilinear data completion and de-noising based on
  tensor-svd.
\newblock In \emph{Proceedings of the IEEE conference on computer vision and
  pattern recognition}, pages 3842--3849, 2014.

\bibitem[Golub and Van~Loan(2013)]{golub2013matrix}
G.H. Golub and C.F. Van~Loan.
\newblock \emph{Matrix Computations}.
\newblock Johns Hopkins Studies in the Mathematical Sciences. Johns Hopkins
  University Press, 2013.
\newblock ISBN 9781421407944.
\newblock URL \url{https://books.google.it/books?id=X5YfsuCWpxMC}.

\bibitem[Liu et~al.(2018)Liu, Chen, and Zhu]{liu2018improved}
Yipeng Liu, Longxi Chen, and Ce~Zhu.
\newblock Improved robust tensor principal component analysis via low-rank core
  matrix.
\newblock \emph{IEEE Journal of Selected Topics in Signal Processing},
  12\penalty0 (6):\penalty0 1378--1389, 2018.

\bibitem[Lu(2018)]{lu2018tensor}
Canyi Lu.
\newblock Tensor-tensor product toolbox.
\newblock \emph{arXiv preprint arXiv:1806.07247}, 2018.

\bibitem[Zhang and Aeron(2016)]{zhang2016exact}
Zemin Zhang and Shuchin Aeron.
\newblock Exact tensor completion using t-svd.
\newblock \emph{IEEE Transactions on Signal Processing}, 65\penalty0
  (6):\penalty0 1511--1526, 2016.

\bibitem[Martin et~al.(2013)Martin, Shafer, and LaRue]{martin2013order}
Carla~D Martin, Richard Shafer, and Betsy LaRue.
\newblock An order-p tensor factorization with applications in imaging.
\newblock \emph{SIAM Journal on Scientific Computing}, 35\penalty0
  (1):\penalty0 A474--A490, 2013.

\bibitem[Boyd et~al.(2004)Boyd, Boyd, and Vandenberghe]{boyd2004convex}
Stephen Boyd, Stephen~P Boyd, and Lieven Vandenberghe.
\newblock \emph{Convex optimization}.
\newblock Cambridge university press, 2004.

\bibitem[Watson(1992)]{watson1992characterization}
G~Alistair Watson.
\newblock Characterization of the subdifferential of some matrix norms.
\newblock \emph{Linear algebra and its applications}, 170:\penalty0 33--45,
  1992.

\bibitem[Horn and Johnson(2012)]{horn2012matrix}
Roger~A Horn and Charles~R Johnson.
\newblock \emph{Matrix analysis}.
\newblock Cambridge university press, 2012.

\bibitem[Schur(1911)]{schur1911bemerkungen}
Jssai Schur.
\newblock Bemerkungen zur theorie der beschr{\"a}nkten bilinearformen mit
  unendlich vielen ver{\"a}nderlichen.
\newblock \emph{Journal f{\"u}r die reine und angewandte Mathematik (Crelles
  Journal)}, 1911\penalty0 (140):\penalty0 1--28, 1911.

\end{thebibliography}

\end{document}